\documentclass[runningheads]{llncs}
\usepackage[T1]{fontenc}
\usepackage{graphicx}
\usepackage{parskip}
\usepackage{enumitem}
\usepackage{url}
\usepackage{hyperref}
\usepackage{booktabs} 
\usepackage{tablefootnote}

\usepackage{subcaption} 

\usepackage{tikz}
\usetikzlibrary{shapes.geometric, arrows, decorations.markings, shadings}
\usetikzlibrary{arrows.meta,positioning,decorations.pathmorphing}

\tikzstyle{startstop} = [rectangle, rounded corners, minimum width=3cm, minimum height=1cm,text centered, draw=black, fill=red!30]
\tikzstyle{process} = [rectangle, minimum width=3cm, minimum height=1cm, text centered, draw=black, fill=blue!30]
\tikzstyle{arrow} = [thick,->,>=stealth]

\usepackage{algorithm}
\usepackage{algorithmic}
\usepackage{amsfonts} 
\usepackage{amsmath}

\usepackage{amsmath}
\usepackage{bm}
\usepackage{xspace}
\usepackage{xcolor}

\usepackage{multirow}
\usepackage{caption}
\newcommand{\R}{\ensuremath{\mathbb{R}}\xspace}

\newcommand{\x}{\ensuremath{\mathbf{x}}\xspace}

\newcommand{\y}{\ensuremath{\mathbf{y}}\xspace}

\newcommand{\Lmc}{\ensuremath{\mathcal{L}}\xspace}

\newcommand{\tran}{\ensuremath{\mathit{\tran}}}

\newcommand{\Agg}{\textit{\Agg}}
\newcommand{\HiT}[1]{\textbf{\textit{x}}_{#1}}

\newcommand{\hit}{\text{HiT}\xspace}
\newcommand{\ont}{\text{OnT(w/o r)}\xspace}
\newcommand{\ontr}{$\text{OnT}$\xspace}

\newcommand{\Omc}{\ensuremath{\mathcal{O}}\xspace}

\newcommand{\Imc}{\ensuremath{\mathcal{I}}\xspace}

\newcommand{\NC}{\ensuremath{\mathsf{N_C}}\xspace}
\newcommand{\NR}{\ensuremath{\mathsf{N_R}}\xspace}

\newcommand{\NI}{\ensuremath{\mathsf{N_I}}\xspace}

\newcommand{\boxsqel}{Box$^2$EL\xspace}
\begin{document}
\title{Language Models as Ontology Encoders\thanks{granted by PSRC projects OntoEm (EP/Y017706/1)
and ConCur (EP/V050869/1)}}

\author{Hui Yang\inst{1} \and
Jiaoyan Chen\inst{1} \and
Yuan He\inst{2,4}\thanks{Work done prior to joining Amazon} \and
Yongsheng Gao\inst{3} \and
Ian	Horrocks\inst{4}
}
% \author{Hui Yang\inst{1}\orcidID{0000-0002-4262-4001} \and
% Jiaoyan Chen\inst{1}\orcidID{0000-0003-4643-6750} \and
% Yuan He\inst{2,4}\orcidID{0000-0002-4486-1262}\thanks{Work done prior to joining Amazon} \and
% Yongsheng Gao\inst{3}\orcidID{0000-0002-3468-2930} \and
% Ian	Horrocks\inst{4}\orcidID{0000-0002-2685-7462}
% }

\authorrunning{Hui and Jianyan, et al.}

\institute{The University of Manchester \\ \email{\{hui.yang-2, jiaoyan.chen\}@manchester.ac.uk} \and
Amazon \email{lawhy@amazon.com} \and
SNOMED International \email{yga@snomed.org} \and
University of Oxford \email{Ian.Horrocks@cs.ox.ac.uk}
}

\maketitle              % typeset the header of the contribution
\begin{abstract}
OWL (Web Ontology Language) ontologies which are able to formally represent complex knowledge and support semantic reasoning have been widely adopted across 
various domains such as healthcare and bioinformatics. 
Recently, ontology embeddings 
have gained wide attention due to their potential to infer plausible new knowledge and approximate complex reasoning. 
However, existing methods face notable limitations: geometric 
model-based embeddings typically overlook valuable textual 
information, resulting in suboptimal performance, 
while the approaches that incorporate text, which are often based on language models, fail to preserve the logical structure. 
In this work, we propose a new ontology embedding method \ontr, which tunes a Pretrained Language Model (PLM) via geometric modeling in a hyperbolic space for effectively incorporating textual labels and simultaneously preserving class hierarchies and other logical relationships of Description Logic $\mathcal{EL}$. 
Extensive experiments on four real-world ontologies show that 
\ontr consistently outperforms the baselines including the state-of-the-art 
across both tasks of prediction and inference of axioms. 
\ontr also demonstrates strong potential in real-world applications, indicated by its robust transfer learning abilities and effectiveness in real cases of discovering new axioms in SNOMED CT construction. Data and code are available at \url{https://github.com/HuiYang1997/OnT}.
    
\end{abstract}

\section{Introduction}

Ontologies of Web Ontology Language (OWL) can represent explicit, formal, and shared knowledge of a domain, supporting complex knowledge by incorporating Description Logic (DL) axioms \cite{DBLP:conf/birthday/BaaderHS05,DBLP:journals/ijinfoman/Fitz-GeraldW10a}. 
These ontologies have become indispensable in domains requiring precise semantic representations; typical examples include the Gene Ontology (GO) \cite{ashburner2000gene} in bioinformatics and SNOMED CT \cite{donnelly2006snomed} in healthcare. With the emergence of neural representation learning techniques \cite{bengio2013representation}, there has been growing interests in developing embedding approaches for ontologies that can encode their entities (which include concepts, roles and instances) as numerical vectors while effectively preserving their structural and semantic properties within the vector space for supporting different downstream tasks of prediction, (approximate) inference, retrieval and so on, usually in combination with other machine learning and statistical methods \cite{chen2024ontology,kulmanov2021semantic}.

Despite significant advancements, the current methods---which can be divided into two types: \textit{geometric model-based} and \textit{language model-based}---still have distinct shortcomings.
\begin{enumerate}[leftmargin=*]
    \item \textit{Geometric Model-Based Methods:}  
    These methods represent ontology entities as geometric objects, 
    such as instances as points and concepts as areas, to construct 
    a geometric model of the target ontology \cite{chen2024ontology}. 
    For example, the early method ELEM \cite{kulmanov_embeddings_2019} 
    represents concepts as balls, while more recent methods like 
    BoxEL~\cite{xiong_faithiful_2022}, 
    Box$^2$EL~\cite{DBLP:conf/www/JackermeierC024}, and 
    TransBox~\cite{yang2025transbox} represents concepts as boxes. 
   Geometric model-based methods preserve logical relationships by translating DL operators into geometric operations—such as representing concept subsumption as area inclusion and conjunction as intersection—thereby supporting reasoning in the vector space.
   %(approximate) axiom inference based on deductive reasoning. 
    However, they mostly neglect valuable textual information, such as entity labels that are common in real-world ontologies.
   This results in suboptimal performance in ontology learning tasks such as axiom prediction, and an inability to embed new entities that are unseen during training—a critical limitation when dealing with dynamic and transfer scenarios.
    %in ontology learning 
    %tasks and limits both inductive reasoning capabilities 
    %(e.g., handling newly introduced entities) and the explainability 
    %provided by textual information.
\item \textit{Language-Model-Based Methods:} 
    These methods, exemplified by  OPA2Vec \cite{DBLP:journals/bioinformatics/SmailiGH19} and OWL2Vec* \cite{DBLP:journals/ml/ChenHJHAH21}, focus on encoding the textual information of ontologies, often following a pipeline which first transforms the axioms and the graph structure into sentences and then tunes a language model to learn entity representations from the sentences \cite{chen2024ontology}. 
    They incorporate both text and formal semantics in the embeddings, which can lead to higher similarities between more related entities \cite{kulmanov2021semantic}, but ignore the preservation of logical relationships, which limits their effectiveness in inference.
    Moreover, most methods generate ontology embeddings using traditional non-contextual word embedding models like Word2Vec, with limited exploration towards the more recent Transformer-based PLMs, which produce layer-specific contextual embeddings rather than general representations. 
    Recently, HiT \cite{hit} has been proposed to bridge this gap by training a PLM with geometric modeling for embedding both concept hierarchies and labels.
    However, HiT is designed for taxonomies, and does not support complex concepts and logical relationships beyond concept subsumption in OWL ontologies.
\end{enumerate}

To address these limitations, we propose \textbf{On}tology \textbf{T}ransformer encoder (\ontr), which integrates the strengths of PLM for contextual text embedding, and geometric modeling in a hyperbolic space for logical structure embedding.
\ontr enables the preservation of logical relationships of Description Logic $\mathcal{EL}$, thus augmenting axiom inference in the vector space.
It effectively incorporates more kinds of semantics for better performance in axiom prediction, and supports the embedding of new entities.

\ontr mainly consists of two steps: (1) Complex concepts (denoted as $C, D$) and roles (denoted by $r$) are embedded into vector representations using a PLM. The embeddings of complex concepts are derived from a verbalization process that generates textual descriptions for these concepts,  while roles are represented as transition functions operating within the space of concept vectors.  
(2) General Concept Inclusion (GCI) axioms of the form $C \sqsubseteq D$ are represented
by regarding them as a hierarchical pre-order $C \prec D$, which is then encoded in a Poincaré ball. 
Moreover, to effectively capture the logical patterns associated with the existential qualifier (i.e., $\exists r.$) and conjunction (i.e., $\sqcap$), \ontr incorporates two specialized loss functions that leverage role embeddings in conjunction with concept embeddings.

Through extensive experiments on real-world ontologies of 
GALEN \cite{rector1996galen}, Gene Ontology (GO) \cite{ashburner2000gene}, 
and Anatomy (Uberon) \cite{mungall2012uberon}, we demonstrate that our 
method \ontr outperforms current state-of-the-art geometric-model or language-model based approaches in both prediction and inference tasks. 
Notably, in terms of the Mean Rank metric, \ontr achieves up to a sevenfold improvement over existing methods, as observed in the prediction task on the GO dataset. Moreover, it exhibits strong transfer learning capabilities, successfully identifying missing and incorrect direct subsumptions in the SNOMED use case, which highlights the practical potential of our approach for real-world ontology applications.

\section{Related Work}
\noindent\textbf{Geometric model-based methods} encode ontologies by representing their concepts and instances as geometric objects in vector spaces and their roles (i.e., binary relations) as specific geometric relationships between these objects. These methods construct an (approximate) geometric model of the ontology, interpreting logical relationships as geometric meanings. For example, the subsumption of concepts can be understood as the set-inclusion of corresponding geometric objects. 
Various geometric representations have been explored for representing concepts, including boxes (TransBox \cite{yang2025transbox}, \boxsqel \cite{DBLP:conf/www/JackermeierC024}, BoxEL \cite{xiong_faithiful_2022}, ELBE \cite{peng_description_2022}), balls (ELEM \cite{kulmanov_embeddings_2019}, EMEM++ \cite{DBLP:conf/aaaiss/MondalBM21}), cones \cite{DBLP:conf/nips/GargISVKS19,zhapa2023cate}, and fuzzy sets \cite{tang2022falcon}. The most common way of representing relations is using transition functions defined by the addition of a given vector. 
Among these approaches, box-based methods have gained prominence due to their closure under intersection --- the intersection of two boxes yields another box --- enabling them to naturally capture concept conjunctions through geometric operations. In contrast, other geometric representations lack this crucial property, limiting their expressiveness for certain logical operations. 
The majority of existing methods focus on $\mathcal{EL}$-family ontologies, with notable exceptions of catE \cite{zhapa2023cate} and FALCON \cite{tang2022falcon}, which provide embeddings for $\mathcal{ALC}$-ontologies.

\textbf{Language model-based methods} originated from early approaches utilizing word embeddings like Word2Vec (which are widely regarded as a kind of neural language models), such as OPA2Vec \cite{DBLP:journals/bioinformatics/SmailiGH19} and OWL2Vec* \cite{DBLP:journals/ml/ChenHJHAH21}. 
They generate embeddings for ontology entities by fine-tuning a word embedding model with the ontology's information, and then apply these embeddings to downstream tasks via an additional, separated prediction model such as a binary classifier. More recently, inspired by the rapid advancement of PLMs based on Transformer architectures \cite{DBLP:conf/nips/VaswaniSPUJGKP17}, a variety of PLM-based approaches such as SORBET and BERTSubs \cite{DBLP:conf/om2/AmirBEEBNZ23,DBLP:journals/www/ChenHGJDH23,DBLP:conf/semweb/GosselinZ23,DBLP:conf/aaai/0008CA022,DBLP:journals/corr/abs-2403-17216} have been developed for ontology-related tasks, particularly in the context of ontology completion and alignment. However, these methods jointly fine-tune a PLM and an additional layer that is specific to a downstream task. Thus they do not yield general embeddings that are applicable across different tasks.
Furthermore, all language model-based methods---whether based on Word2Vec or transformers---fail to capture logical structures such as the transitivity of subsumption relationships, thereby preventing direct inference within the vector space.

Recently, He et al. proposed HiT \cite{hit} --- a method that combines language models with hierarchical embedding techniques in hyperbolic spaces to embed taxonomies that consist of hierarchical structures of named concepts. However, this approach overlooks role embeddings and the logical operations that construct complex concepts from basic ones, which are prevalent in real-world ontologies. In this work, we address this limitation with role embeddings and specialized loss functions that capture the logical operators used to build complex concepts from fundamental ones.

We exclude work on Knowledge Graphs such as KG-BERT~\cite{DBLP:journals/corr/abs-1909-03193} and KEPLER~\cite{DBLP:journals/tacl/WangGZZLLT21}, as our focus is on OWL ontologies, which use Description Logic to model conceptual knowledge—fundamentally different from relational fact-based Knowledge Graphs.

\section{Preliminary}
\subsection{Ontology}
OWL ontologies employ sets of statements, known as axioms, to represent and reason about concepts (unary predicates) and roles (binary predicates). 
In this work, we focus on $\mathcal{EL}$-ontologies, which are investigated by most existing geometric embedding methods. These ontologies strike a balance between expressivity and reasoning efficiency, making them widely applicable~\cite{DBLP:journals/ai/BaaderG24}.
Consider the disjoint sets $\NC= \{A,B,\ldots\}$, $\NR = \{r, t, \ldots\}$, and $\NI=\{a,b,\ldots\}$, 
representing \emph{concept names} (a.k.a. \emph{atomic or named concepts}), \emph{role names}, and \emph{individual names}, 
respectively. \emph{$\mathcal{EL}$-concepts} (complex concepts) are defined recursively from these elements as 
$\top ~|~ \bot ~| ~A~ | ~ C \sqcap D~ | ~\exists r. C ~| \left\{ a \right\}$. 
An \emph{$\mathcal{EL}$-ontology} is a finite collection of TBox axioms like $C\sqsubseteq D$, 
and ABox axioms like $A(a)$ and $r(a, b)$. 
Note that, through the paper, we denote by atomic concepts as $A, B$, and any 
$\mathcal{EL}$-concepts as $C, D$.

\begin{example}\label{exp:ont}
Given atomic concepts $\textit{Teacher}, \textit{Student}, \textit{Class}$, 
roles $\textit{teach}, \textit{hasClass}$ and $\textit{study}$, and individuals $\textit{Dr.Smith}, \textit{Emma}$, there is a small $\mathcal{EL}$-ontology 
 composed of TBox axioms: 
\(\textit{Person} \sqcap \exists \textit{teach.Class}\sqsubseteq \textit{Teacher}, \ \textit{Person} \sqcap \exists \textit{study.Class} \sqsubseteq \textit{Student},\)
and ABox axioms:
\(\textit{Teacher}(\textit{Dr.Smith}), \textit{hasClass}(\textit{Emma}, \textit{Math101}).\)
\end{example}

\noindent\textbf{Normalization of $\mathcal{EL}$-ontology}.
In this work, we focus on the TBox part. 
Note that Abox axioms can be transformed into  equivalent TBox axioms by treating instances as classes~\cite{DBLP:conf/www/JackermeierC024}. 
An $\mathcal{EL}$-ontology $\mathcal{O}$ is normalized if all its (TBox) axioms are 
of one of the following forms:
\begin{equation}\label{eq:normalized_ontology}
A\sqsubseteq B,\quad A_1\sqcap A_2\sqsubseteq B,\quad A\sqsubseteq \exists r. B,\quad \exists r. B\sqsubseteq A.
\end{equation}
For simplicity, we refer to these four types of normalized axioms as NF1-NF4 (where NF denotes normalized form), respectively. It is worth noting that most existing geometric embedding methods are exclusively applicable to normalized ontologies. 
Any $\mathcal{EL}$-ontology can be transformed into a set of normalized axioms~\cite{DBLP:conf/ijcai/BaaderBL05} by introducing new atomic concepts along with corresponding names, as illustrated in the following example.

\begin{example}
To normalize the axiom $\textit{Person} \sqcap \exists \textit{teach.Class} \sqsubseteq \textit{Teacher}$ from Example~\ref{exp:ont}, we introduce a new atomic concept \( N_1 \equiv \exists \mathit{teach.Class} \). This transfers the original axiom into three normalized axioms: 
\[
\mathit{Person} \sqcap N_1 \sqsubseteq \mathit{Teacher}, \quad N_1 \sqsubseteq \exists \mathit{teach.Class}, \quad \text{and} \quad \exists \mathit{teach.Class} \sqsubseteq N_1.
\]
Here, the newly introduced concept \( N_1 \), derived from \( \exists \mathit{teach.Class} \), can be informally interpreted as “Something that teaches some Class.”
\end{example}

%(see Appendix \ref{app:normalization} for details).

\noindent\textbf{Inference} An \emph{interpretation} $\mathcal{I} = (\Delta^\mathcal{I}, \ \cdot^\mathcal{I})$ 
comprises a non-empty set~$\Delta^\mathcal{I}$ and a function~$\cdot^\Imc$ that 
maps each $A \in \NC$ to $A^\mathcal{I} \subseteq \Delta^\mathcal{I}$, 
each $r \in \NR$ to $r^\mathcal{I} \subseteq \Delta^\mathcal{I} \times \Delta^\mathcal{I}$, 
and each $a \in \NI$ to $a^\Imc \in \Delta^\Imc$, with $\bot^\Imc = \emptyset$, 
$\top^\mathcal{I} = \Delta^\mathcal{I}$, and $\{a\}^\Imc = a^\Imc$. 
This function extends to any $\mathcal{EL}^{++}$-concepts as follows:
\[
    (C\sqcap D)^\mathcal{I} = C^\mathcal{I}\cap D^\mathcal{I},\quad
    (\exists r.C)^\mathcal{I} = \left\{a \in \Delta^\mathcal{I} \mid \exists b \in C^\mathcal{I}: (a,b) \in r^\mathcal{I}\right\},
\]

An interpretation $\mathcal{I}$ \emph{satisfies} a TBox axiom $X\sqsubseteq Y$ if $X^\Imc \subseteq Y^\Imc$ for $X,Y$ being two concepts or two role names, or $X$ being a role chain and $Y$ being a role name. It satisfies an ABox axiom $A(a)$ if $a^\Imc \in A^\Imc$ and it satisfies $r(a,b)$ if $(a^\Imc, b^\Imc) \in r^\Imc$. Finally,  $\mathcal{I}$ is a \emph{model} of $\Omc$ if it satisfies every axiom in $\Omc$. An ontology $\Omc$ \emph{entails} an axiom $\alpha$, denoted $\Omc\models\alpha$, if $\alpha$ is satisfied by all models of $\Omc$.

\subsection{Hyperbolic Space}
A $d$-dimensional \emph{manifold} \cite{Lee_2013}, denoted $\mathcal{M}$, 
can be regarded as a hypersurface embedded in a higher $n$-dimensional Euclidean space $\mathbb{R}^n$, 
which is locally equivalent to $\mathbb{R}^d$ around each point $\x\in \mathcal{M}$. 
A \emph{Riemannian manifold} $\mathcal{M}$ is a manifold equipped with a Riemannian metric, enabling the definition of a distance function $d_\mathcal{M}(\mathbf{x}, \mathbf{y})$ for $\mathbf{x}, \mathbf{y} \in \mathcal{M}$.  
\textit{Hyperbolic space}, denoted $\mathbb{H}^n$, is a Riemannian manifold with a constant negative curvature of $-\kappa$ ($\kappa > 0$) \cite{lee2006riemannian}, which can be represented  by the Poincaré ball model whose points are defined by a ``ball'' with radius $1 /\sqrt{\kappa}$:
\(
B^n = \{ \x \in \mathbb{R}^n : \|\x\| < 1 / \sqrt{\kappa} \},
\) 
and the hyperbolic distance between $\x, \y \in B^n$ are defined as
\begin{equation}\label{eq:hyperbolic_distance}
d_\kappa(\x, \y) = \frac{1}{\sqrt{\kappa}} \, \text{arcosh} \left( 1 + \frac{2 \kappa \|\x - \y\|^2}{(1 - \kappa\|\x\|^2)(1 - \kappa\|\y\|^2)} \right).
\end{equation}
% add the definition of scaling product over hyperbolic space
In the Poincaré ball model, the scaling $k\in R$ of a point $\x \in B^n$ is defined as
\begin{equation}\label{eq:scaling_product}
k \odot \x = \tanh(k \cdot \tanh^{-1}(||\x||)) \cdot \frac{\x}{||\x||}
\end{equation}

\section{Methodology}
In this section, we present our method, \ontr, for embedding a given $\mathcal{EL}$-ontology.  
\ontr consists mainly of three parts: 
\begin{enumerate}[leftmargin=*]
    \item Embedding any $\mathcal{EL}$-concepts (atomic or complex ones) as points in hyperbolic spaces using PLMs and verbalizations (Section~\ref{sec:concept_embedding});
    \item Embedding roles as rotations over hyperbolic spaces (Section~\ref{sec:role_embedding}), which allows \ontr for capturing the logical structures of existential qualifications $\exists r$ (Proposition \ref{prop:distance-preserving}) and demonstrates improved performance as evidenced by evaluations on real-world ontologies; 
    \item Training the embeddings using the Poincaré ball model (Section~\ref{sec:training}) by regarding the axioms as hierarchical relationships between complex concepts. 
\end{enumerate}
%By utilizing role embeddings,  \ontr can capture the logical structures of existential qualifications $\exists r.$ (Proposition \ref{prop:distance-preserving}) and demonstrates improved performance as evidenced by evaluations on real-world ontologies .
It is worth noting that the role embedding component can be omitted to yield a simplified variant, referred to as \ont.

\subsection{Verbalisation-based Concept Embedding}\label{sec:concept_embedding}

Given an ontology $\mathcal{O}$, we assume that each atomic concept $A$ and 
role $r$ occurring in $\mathcal{O}$ is associated with a textual description, 
typically its name or definition, denoted as $\mathcal{V}(A)$ and $\mathcal{V}(r)$, 
respectively. For instance, we may have $\mathcal{V}(A) = \text{``Father''}$ and
$\mathcal{V}(r) = \text{``is parent of''}$.

Based on these descriptions of atomic concepts and roles, we systematically 
generate a natural language description for each complex concept $C$ appearing 
in the ontology $\mathcal{O}$, denoted as $\mathcal{V}(C)$. 
For $\mathcal{EL}$-ontologies, we generate these descriptions according to the 
following compositional rules:
\[
\mathcal{V}(C\sqcap D) = ``\mathcal{V}(C) \text{ and } \mathcal{V}(D)",\quad 
\mathcal{V}(\exists r. C) = `` \text{ something that }\mathcal{V}(r) \text{ some } \mathcal{V}(C)".
\]
For example, we will have
\(\mathcal{V}(\text{Person}\sqcap\text{Student}) = \text{``person and student''}\), and \(\mathcal{V}(\exists \text{isParentOf}.\text{Person}) = \text{``something that is parent of some person''}\). 

With the verbalization approach described above, we can embed any complex $\mathcal{EL}$-concept $C$ by applying 
language models to its textual description $\mathcal{V}(C)$ and mapping the result to a point in hyperbolic space as in HiT \cite{hit}. Specifically, this is achieved by encoding sentences using a BERT model with mean pooling, after which the resulting embeddings are re-trained in hyperbolic space.
The final embedding of $C$ is denoted as $\HiT{C}$. 

\subsection{Logic-aware Role Embedding}\label{sec:role_embedding}

In the above verbalization process, the role semantics is integrated into the concept verbalizations. 
However, this do not provide individual roles embeddings, which could restrict the capability of
handling logical patterns involving roles and impairing the reasoning. 
For instance, we can not guarantee the preservation of deductive patterns such as the monotonicity of existential restrictions: 
if $A\sqsubseteq B$, then $\exists r. A\sqsubseteq \exists r. B$.

To address these limitations, we propose to explicitly incorporate role embeddings by interpreting a role $r$ as 
a function $f_r$ over hyperbolic space.  In details, for each complex concept of the form $\exists r. D$, \ontr introduces an alternative representation:
\(f_{r}(\HiT{D})\), 
which complements the verbalization-based embedding $\HiT{\exists r. D}$. 
Here, $f_r$ is a role-specific transformation function. We will encourage the two embeddings \(f_{r}(\HiT{D})\) and $\HiT{\exists r. D}$ 
to be identical by introducing an extra loss term in the training process.

\begin{figure}[t]
    \centering
    \begin{tikzpicture}[scale=1.8]
        % Apply rotation to the entire picture
        \begin{scope}[rotate=-30]
            % Create a shading for the Poincaré ball to show distance distortion
            \shade[inner color=white!00, outer color=black!70, opacity=0.7] (0,0) circle (1.2);
            
            % Draw original point D
            \fill[blue] (0.3,0.5) circle (0.03);
            
            % Draw the rotated point X_C
            \fill[red] (-0.7,0.42) circle (0.03);
            
            % Add the verbalized point for \HiT{\exists r. D}
            \fill[blue] (-0.8,0.3) circle (0.03);
            
            % Origin point
            \fill[black] (0,0) circle (0.02);
            
            % Add some hyperbolic geodesics (curved lines) to show structure
            \draw[dashed, gray] (0,0) -- (0.3,0.5);
            \draw[dashed, gray] (0,0) -- (-0.5,0.3);
            \draw[dashed, red!70!black, decoration={markings, mark=at position 0.5 with {\arrow{>}}}, postaction={decorate}] (-0.5,0.3) -- (-0.7,0.42);
            
            % Draw a curved arc to illustrate the rotation angle (adjusted position)
            \draw[dashed, red!70!black, decoration={markings, mark=at position 0.5 with {\arrow{>}}}, postaction={decorate}] (0.3,0.5) arc (60:150:0.583);
        \end{scope}
        
        % Text labels are placed outside the scope to prevent rotation
        \node at (0,-0.7) {Poincaré Ball $B^2$};
        
        % Calculate rotated coordinates for text placement
        % For a point (x,y) rotated -30 degrees, new coords are (x*cos(-30)-y*sin(-30), x*sin(-30)+y*cos(-30))
        \node[blue, right, font=\large] at (0.55, 0.3) {$\HiT{D}$};
        \node[red!70!black, above, font=\large] at (-0.3, 0.75) {$f_r(\HiT{D})$};
        \node[blue, below left, font=\large] at (-0.4, 0.7) {$\HiT{\exists r. D}$};
        \node[below, font=\large] at (0,0) {$O$};
        
        % Label for the scaling arrow
        \node[red!70!black, sloped, font=\large] at (-0.22, 0.65) {$k_r$};
        
        % Rotation angle label
        \node[red!70!black, font=\large] at (0.1, 0.44) {$\theta_r$};
        
        % Add rotation label
        % \node[blue, text width=2cm] at (-0.1,-0.4) {$f_r = R(\Theta_r)$};
    \end{tikzpicture}
    \caption{Illustration of $f_r$ in a two-dimensional hyperbolic space.}
    \label{fig:poincare-rotation}
\end{figure}
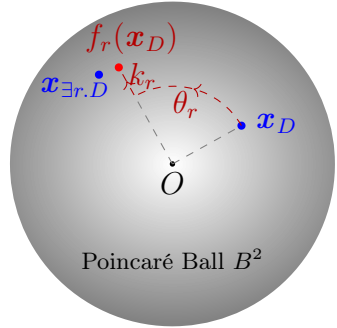

In our implementation, we define $f_r$ as a composition of rotations and scaling 
operations in hyperbolic space. Specifically, the $f_r$ is defined by 
(see Fig. \ref{fig:poincare-rotation} for an illustration):
\begin{equation}\label{eq:fr}
f_r(\mathbf{v}) = k_r \odot (R(\Theta_r) \cdot \mathbf{v}),
\end{equation}
where $k_r \in \R$ is a role-specific scaling factor, 
$\Theta_r = (\theta_{r}^1, \theta_{r}^2, \ldots, \theta_{r}^m)\in \mathbb{R}^m$ 
is a role-specific rotation angle, and $\mathbf{v} \in \mathbb{H}^{2m}$ is 
a point in hyperbolic space. 
Here, the $\odot$ operation represents the scaling product over 
hyperbolic space $\mathbb{H}^{2m}$, 
which ensures the scaled embeddings are still in hyperbolic space. 
However, for rotations, we could directly apply the same rotations as the Euclidean space as we use the Poincaré ball models 
for the representation of hyperbolic space. Specifically, we use the rotation matrix 
$R(\Theta_r) \in \mathbb{R}^{2m \times 2m}$ over the space $\mathbb{H}^{2m}$ defined as a product of two-dimensional rotations 
of the following form:
\[
R(\Theta_r) = \begin{bmatrix}
    R(\theta_{r}^1) & \textbf{0} & \cdots & \textbf{0} \\
    \textbf{0} & R(\theta_{r}^2) & \cdots & \textbf{0} \\
    \vdots & \vdots & \ddots & \vdots \\
    \textbf{0} & \textbf{0} & \cdots & R(\theta_{r}^m)
\end{bmatrix},  
\text{ where }
R(\theta_r) = \begin{bmatrix}
\cos(\theta_r) & -\sin(\theta_r) \\
\sin(\theta_r) & \cos(\theta_r)
\end{bmatrix}.
\]

\begin{figure}[t]
  \centering
  \begin{subfigure}[b]{0.4\textwidth}
    \includegraphics[width=0.8\textwidth]{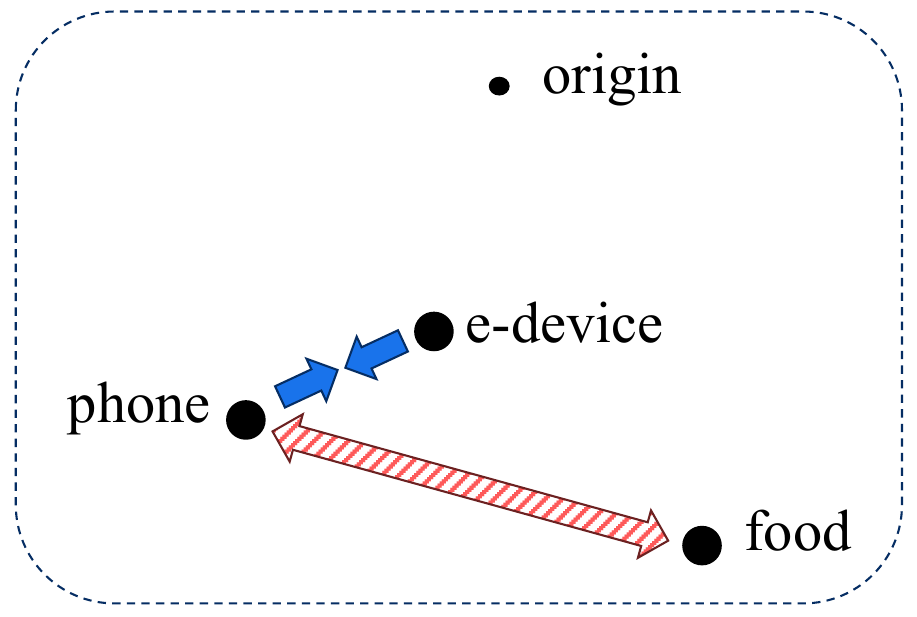}
    \caption{Hierarchical Contrastive Loss}
    \label{fig:loss1}
  \end{subfigure}
  \ \ \ 
  \begin{subfigure}[b]{0.4\textwidth}
    \includegraphics[width=0.8\textwidth]{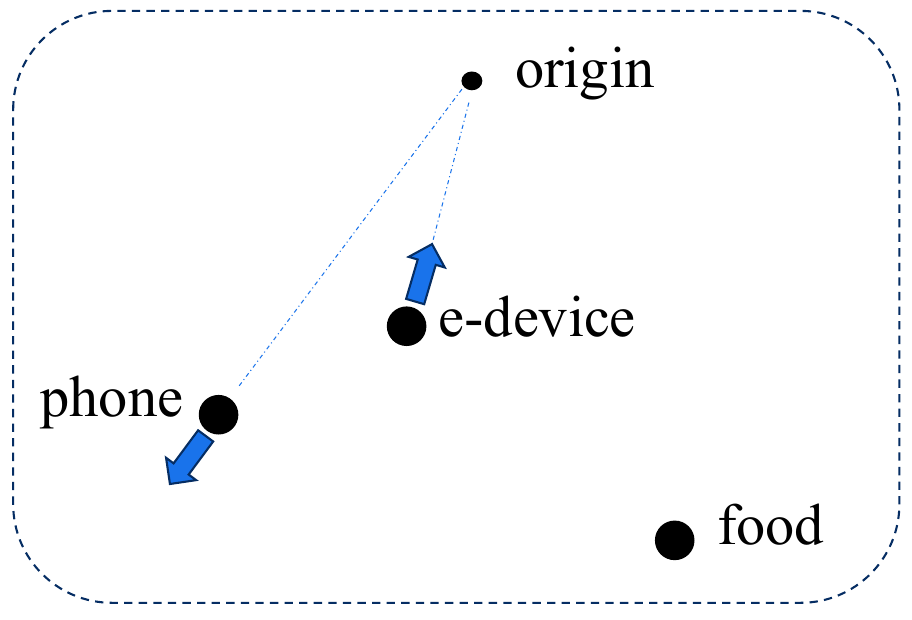}
    \caption{Centripetal Loss}
    \label{fig:loss2}
  \end{subfigure}
  \caption{Illustration of impact of hierarchy Loss $\mathcal{L}_{\prec}$ during
training.}
  \label{fig:loss}
\end{figure}

\subsection{Training}\label{sec:training}

\subsubsection{Hierarchy Loss} 
We interpret subsumption axioms $C\sqsubseteq D$ in the ontology $\mathcal{O}$ as partial-order relationships between their embeddings: 
\(\HiT{C} \prec \HiT{D}\). Then, following the approach of HiT \cite{hit}, we encode these partial-order relationships using a 
Poincaré embedding model \cite{poincare} using a hierarchical loss defined by the hyperbolic distance. 
% The training and inference procedures for this model are detailed below.
The loss function $\mathcal{L}_{\prec}(\HiT{C}\prec \HiT{D})$ consists of two parts:
\begin{enumerate}[leftmargin=*]
    \item \textit{Hierarchical Contrastive Loss}: This loss encourages embeddings of related concepts to be closer to each other than to negative samples:
    \[
    \mathcal{L}_{\textit{contrast}}(\HiT{C}\prec \HiT{D}) = 
    \max(0, d_\kappa(\HiT{C}, \HiT{D}) - d_\kappa(\HiT{C}, \HiT{D_{\text{neg}}}) + \alpha),
    \] 
    where $D_{\text{neg}}$ represents a randomly sampled concept that composes a negative example with $C$ and $\alpha$ is a margin hyperparameter.

    \item \textit{Centripetal Loss}: This loss enforces that parent concepts are embedded closer to the origin than their children in the hyperbolic space. 
    Let $\|\textbf{\textit{x}}\|_\kappa$ denote the hyperbolic distance from a point $\textbf{\textit{x}}\in \mathbb{H}^n$ to the origin (also known as the hyperbolic norm). The centripetal loss is defined as:
    \[
    \mathcal{L}_{\textit{centri}}(\HiT{C}\prec \HiT{D}) = \max(0, \|\HiT{D}\|_\kappa - \|\HiT{C}\|_\kappa + \beta),
    \]
    where $\beta$ is a margin hyperparameter. This constraint geometrically reinforces the hierarchical structure by positioning more general concepts (parents) closer to the center of the hyperbolic space.
\end{enumerate}
The overall hierarchy loss is defined as the sum of these two loss components:
\begin{equation}\label{eq:training_loss}
\mathcal{L}_{\prec}(\HiT{C}\prec \HiT{D}) = \mathcal{L}_{\textit{contrast}}(\HiT{C}\prec \HiT{D}) + \mathcal{L}_{\textit{centri}}(\HiT{C}\prec \HiT{D}).
\end{equation}

The effect of the hierarchy loss on embedding updates is illustrated in Figure~\ref{fig:loss}, where the positive pair is \( C = \textit{"phone"} \) and \( D = \textit{"e-device"} \), and the negative example is \( D_{\text{neg}} = \textit{"food"} \). The contrastive loss encourages the embeddings of \( C \) and \( D \) to be close, while pushing \( C \) and \( D_{\text{neg}} \) farther apart, as shown in Figure~\ref{fig:loss1}. On the other hand, in in Figure~\ref{fig:loss2}, the centripetal loss pulls the parent concept \( D \) toward the origin, while pushing the child concept \( C \) away from it.

\subsubsection{Loss for role embeddings} The loss for role embeddings aims to align the embeddings 
of $\HiT{\exists r. D}$ with $f_{r}(\HiT{D})$. However, as shown by our 
preliminary experiments, it is not a good choice to directly align the 
embeddings such as introducing a loss defined by their Euclidean distance or hyperbolic, 
i.e., \(||\HiT{\exists r. D} - f_{r}(\HiT{D})||\) or \(d_\kappa(\HiT{\exists r. D}, f_{r}(\HiT{D}))\). 
Instead, we would reuse the hierarchical loss above by interpreting the 
equivalence $\HiT{\exists r. D}\equiv f_{r}(\HiT{D})$ as two partial-order $\HiT{\exists r. D}\prec f_{r}(\HiT{D})$ and $f_{r}(\HiT{D})\prec \HiT{\exists r. D}$. 
Formally, the loss is defined as:
\begin{equation}\label{eq:loss_rotation}
\Lmc_{r}(\exists r. D) =
\frac{1}{2} \Big( \Lmc_{\prec}\big(\HiT{\exists r. D}\prec  f_{r}(\HiT{D})\big) + \Lmc_{\prec}\big(f_{r}(\HiT{D})\prec \HiT{\exists r. D}\big)\Big)
\end{equation}

\subsubsection{Loss for conjunction} This loss is introduced to capture the logical properties of the conjunction $\sqcap$, specifically,  a universally valid axiom $C\sqcap D\sqsubseteq C$. 
It is enough to use the following loss based on the hierarchy loss $\Lmc_{\prec}$: 
\begin{equation}\label{eq:loss_conjunction}
\Lmc_{\sqcap}(C\sqcap D) = \frac{1}{2}\Big(\Lmc_{\prec}(\HiT{C\sqcap D}\prec \HiT{C}) + \Lmc_{\prec}(\HiT{C\sqcap D}\prec \HiT{D})\Big).
\end{equation}

\subsubsection{Training} The final train loss is defined as the sum of the losses defined by Equations 
\eqref{eq:training_loss}, \eqref{eq:loss_rotation}, and \eqref{eq:loss_conjunction} for all axioms $C\sqsubseteq D$, concept $\exists r. D$, and conjunctions $C\sqcap D$ appeared  $\Omc$, respectively.

Finally, with a well-trained \ontr model, we evaluate a new axiom $C\sqsubseteq D$ using the following score with a higher value indicates a higher confidence of the given axioms,  which is defined as a weighted sum of distances:
\begin{equation}\label{eq:scoring}
    s(C\sqsubseteq D) \equiv s(\HiT{C}\prec \HiT{D}) := -(d_\kappa(\HiT{C}, \HiT{D}) + \lambda (\|\HiT{D}\|_\kappa - \|\HiT{C}\|_\kappa))
\end{equation}
where the weight $\lambda$ is determined based on the model's performance on the validation set, higher scores indicate a stronger predicted subsumption relationship between concepts.

We have the following proposition that allows us to control the difference between scores $s(f_r(\HiT{C})\prec f_r(\HiT{D}))$ and $s(\HiT{C}\prec \HiT{D})$ using the scaling factor $k_r$, and thus, capturing the deductive pattern 
$A\sqsubseteq B \Rightarrow \exists r. A\sqsubseteq \exists r. B$.

\begin{proposition}\label{prop:distance-preserving}
    For any $\x, \y\in \mathbb{H}^{2m}$ and rotation matrix $R(\Theta_r)\in \mathbb{R}^{2m\times 2m}$ as defined in Equation~\eqref{eq:fr}, 
    we have $\|\HiT{C}\|_\kappa = \|R(\Theta_r)\cdot \HiT{C}\|_\kappa$ and $d_\kappa(\x, \y) = d_\kappa(R(\Theta_r)\cdot \x, R(\Theta_r)\cdot \y)$. Moreover, we have $s(f_r(\HiT{C})\prec f_r(\HiT{D}))=s(\HiT{C}\prec \HiT{D})$ when $k_r=1$.
\end{proposition}
\begin{proof}
The hyperbolic distance $d_\kappa(\x, \y)$ depends only on the Euclidean norms $\|\x\|$ and $\|\y\|$, as per Equation~\eqref{eq:hyperbolic_distance}. Since the rotation $R(\Theta_r)$ preserves Euclidean norms, it follows that $\|R(\Theta_r) \cdot \mathbf{z}\| = \|\mathbf{z}\|$ for $\mathbf{z} = \x, \y$. Therefore,
\(
d_\kappa(\x, \y) = d_\kappa(R(\Theta_r) \cdot \x, R(\Theta_r) \cdot \y).
\)
By definition, we have $\|\x\|_\kappa = d_\kappa(\x, \mathbf{0})$. Applying this with $\y = \mathbf{0}$, we obtain:
\(
\|\x\|_\kappa = \|R(\Theta_r) \cdot \x\|_\kappa.
\)

Since the score is defined by $d_\kappa(\x, \y)$, $\|\x\|_\kappa$, and $\|\y\|_\kappa$, and given $f_r(\mathbf{z}) = R(\Theta_r) \cdot \mathbf{z}$ when $k_r = 1$, we conclude that $s(f_r(\HiT{C}) \prec f_r(\HiT{D})) = s(\HiT{C} \prec \HiT{D})$ when $k_r = 1$. This completes the proof.
\end{proof}

\section{Evaluation}

\subsection{Experiment Setting}

The evaluation is mainly concentrated on two tasks: 
axiom prediction (Section~\ref{sec:prediction}) and inference (Section~\ref{sec:inference}). 
The prediction and inference tasks focus on identifying missing axioms; however, the prediction task addresses arbitrary axioms, while the inference task focuses specifically on axioms that can be logically derived from the given ontologies.
We also evaluate the performance of our method in different scenarios such as transfer learning, ablation study, and over real cases in Section~\ref{sec:case_study}.

\urldef{\mowlURL}\url{https://mowl.readthedocs.io/en/latest/_modules/mowl/ontology/normalize.html#ELNormalizer}

\noindent
\textbf{Datasets}  
We adopt three real-world ontologies --- 
GALEN \cite{rector1996galen}, 
the Gene Ontology (GO) \cite{ashburner2000gene}, 
and Anatomy (Uberon) \cite{mungall2012uberon}. Following the prior research 
\cite{DBLP:conf/www/JackermeierC024,yang2025transbox}, we keep only the $\mathcal{EL}$ part and use their normalized  versions. 
For the prediction task, the training, validation, and testing data are generated by a random 80/10/10 split of the ontology axioms. 
For the inference task, we use the whole ontology as the training data, and all the inferred axioms of NF1 as the testing data, 
and 1000 randomly selected inferred NF1 subsumptions as validation data. 
The data statistics are shown in Table~\ref{tab:dataset-splits}.
Note that we developed our own ontology normalization implementation 
rather than using the existing implementation in ELEM \cite{kulmanov_embeddings_2019}, mOWL \cite{zhapa2023mowl}, and DeepOnto \cite{he2024deeponto} as it (1) does not name the concepts introduced during normalization, and (2) sometimes produces logically inconsistent axioms\footnote{For example, 
among the normalized axioms of GALEN, we find the axiom \(\exists_{\textit{hasQuantity}}\ \textit{BNFSection13\_3}\sqsubseteq \textit{Tobacco}\), which contradicts the original ontology where \textit{BNFSection13\_3} appears only in \(\textit{BNFSection13\_3} \sqsubseteq \textit{BNFChapter13Section}\).} due to some bug in calling the jcel normalizer.

\begin{table}[t]
    \centering
    \caption{Normalized Dataset Statistics (Train/Val/Test for prediction task).}
    \label{tab:dataset-splits}
    \begin{tabular}{lrrr}
    \toprule
    \textbf{Axioms} & \textbf{GALEN} & \textbf{GO} & \textbf{ANATOMY} \\
    \midrule
    \textbf{NF1} & 25,610/3,200/3,203 & 116,751/14,593/14,596 & 41,764/5,220/5,222 \\
    \textbf{NF2} & 11,679/1,459/1,462 & 24,097/3,011/3,014 & 12,336/1,542/1,543 \\
    \textbf{NF3} & 25,299/3,161/3,165 & 238,899/29,861/29,865 & 39,766/4,970/4,972 \\
    \textbf{NF4} & 6,287/785/788 & 81,948/10,243/10,245 & 7,586/947/951 \\
    \midrule
    \textbf{Total} & 68,875/8,605/8,618 & 461,695/57,708/57,720 & 101,452/12,679/12,688 \\
    \midrule
    \textbf{Inferred(NF1)} & 335,002& 1,184,380& 225,330 \\
    \bottomrule
    \end{tabular}
    \end{table}

\noindent
\textbf{Baselines} 
Our study systematically compares our proposed methods with 
established approaches that provide general ontology embeddings, 
 with particular emphasis on geometric embedding methods, including 
 Box$^2$EL~\cite{DBLP:conf/www/JackermeierC024}, BoxEL~\cite{xiong_faithiful_2022}, 
 TransBox~\cite{yang2025transbox}, ELBE~\cite{peng_description_2022}, and ELEM~\cite{kulmanov_embeddings_2019}. 
We exclude catE and FALCON as catE cannot handle unseen complex concepts that appear in our experimental settings, and FALCON's implementation is not publicly available. 
Additionally, we benchmark against HiT~\cite{hit}, 
a language model-based method limited to taxonomic structures (i.e., NF1 axioms), and 
two classic none contextual word embedding-based methods --- OPA2Vec \cite{DBLP:journals/bioinformatics/SmailiGH19} and OWL2Vec* \cite{DBLP:journals/ml/ChenHJHAH21}. 
We also include a simplified version of \ontr, denoted as \ont, which omits role embeddings and is trained using only the loss of Eq.~\ref{eq:training_loss}. 
We ignored other PLM fine-tuning-based methods like BERTSub \cite{DBLP:journals/www/ChenHGJDH23} whose embeddings are coupled to a task-specific layer without generality towards different tasks.

\noindent
\textbf{Evaluation Metrics} 
Consistent with the established literature~\cite{DBLP:conf/www/JackermeierC024,kulmanov_embeddings_2019,peng_description_2022,xiong_faithiful_2022,yang2025transbox}, 
we evaluate ontology embedding performance using various 
ranking-based metrics on the testing set. 
We rank 
candidates according to the score function defined in 
Eq.~\ref{eq:scoring}, where higher scores indicate more probable candidates. 
% Based on axiom types, we evaluate four distinct prediction tasks corresponding to normalized axiom forms: 
% $A\sqsubseteq ?B$ (NF1), $A_1\sqcap A_2\sqsubseteq ?B$ (NF2), $?A\sqsubseteq \exists r. B$ (NF3), and $\exists r. A\sqsubseteq ?B$ (NF4).
%
To comprehensively assess different methods, we track the rank of correct answers and report performance through several standard metrics: Hits@k (H@k) for \( k \in \{1, 10, 100\} \), mean reciprocal rank (MRR), and mean rank (MR). % median rank (Med), and area under the ROC curve (AUC).

\noindent
\textbf{Experimental Protocol} 
We mainly use all-MiniLM-L12-v2 (33.4M) as the underlying language model for \ontr and HiT.  The influence of different language models is presented in Section \ref{sec:ablation_study}. 
We trained  \ontr and HiT for 1 epoch, and with 1 negative sample for each given axiom, as we found that it would be enough to get good performance in the pre-test. 
The embedding vectors for each concept is obtained by performing an average pooling over features of the final layer 
of the language model. For obtaining the vector for $\Theta(r), k_r$ for a given role $r$, we apply an extra linear transformation on the 
embedding of $r$. The margins $\alpha, \beta$ and learning rate $\gamma$ are fixed as default in HiT as $3.0, 0.5, 10^{-5}$, respectively. The weight $\lambda\in\{0, 0.1, \ldots, 1\}$ of the score function in Equation \ref{eq:scoring} is selected based on the best performance on the validation set.

OWL2Vec* and OPA2Vec utilize fine-tuned word embeddings (\url{https://tinyurl.com/word2vec-model})  with the Random Forest classifier for superior performance, except for GO ontology inference tasks, where Logistic Regression is employed due to computational constraints. 
Due to dataset modifications, we also re-implemented all the other geometric embedding models (BoxEL, TransBox, ELBE and ELEM) based on the framework developed by Box$^2$EL \cite{DBLP:conf/www/JackermeierC024} and TransBox \cite{yang2025transbox}. For our implementation, we utilized embedding dimensions $d=200$, explored margin values $\gamma \in \{0, 0.05, 0.1, 0.15\}$ and learning rates $l_r \in \{0.0005, 0.005, 0.01\}$, and trained each model for 5,000 epochs. Optimal hyperparameters were selected based on validation set performance.

\subsection{Prediction Task}\label{sec:prediction}
\begin{table}[t]
    \centering
    \caption{Overall performance of the prediction task across datasets. Values for H@k and MRR are percentages. $k=1/10/100$ for H@k.}
    \label{tab:all_results_pred}
    \begin{tabular}{lccccccccc}
    \toprule
    & \multicolumn{3}{c}{\textbf{GALEN}} & \multicolumn{3}{c}{\textbf{GO}} & \multicolumn{3}{c}{\textbf{ANATOMY}} \\
    \cmidrule(lr){2-4}\cmidrule(lr){5-7}\cmidrule(lr){8-10}
    \textbf{Method} & \textbf{H@k} & \textbf{MRR} & \textbf{MR} & \textbf{H@k} & \textbf{MRR} & \textbf{MR} & \textbf{H@k} & \textbf{MRR} & \textbf{MR} \\
    \midrule
    ELEM & 14/\textbf{50}/68 & 26 & 2,715 & 4/35/68 & 14 & 11,764 & 10/53/78 & 24 & 1,588 \\
    ELBE & 9/37/55 & 18 & 4,661 & 10/30/43 & 18 & 10,236 & 9/41/66 & 20 & 2,672 \\
    BoxEL & 0/0/2 & 0 & 13,824 & 0/0/2 & 0 & 65,846 & 1/2/4 & 2 & 12,257 \\
    Box$^2$EL & 12/38/58 & 21 & 4,593 & 8/43/64 & 19 & 7,975 & 11/39/65 & 20 & 2,828 \\
    TransBox & 11/41/62 & 22 & 2,972 & 8/43/67 & 19 & 7,092 & 9/49/73 & 22 & 1,299 \\
    \hline
     OPA2Vec& 0/1/4& 1& 13,547& 0/1/4 & 0 & 18,493&0/5/17& 2 & 9,537\\
    OWL2Vec*& 0/1/5 & 1 & 13,660 & 0/0/2 & 0 & 19,523 & 0/3/11 & 2 & 10,309 \\
    \hit & \textbf{25}/47/62 & 33 & 2,349 & 36/60/73 & 44 & 15,080 & 19/54/78 & 31 & 722 \\
    \midrule
    \ont & \textbf{26}/46/64 & 33 & 1,546 & \textbf{38}/66/79 & \textbf{48} & 2,209 & \textbf{22}/52/79 & 31 & 628 \\
    \ontr & \textbf{25}/\textbf{50}/\textbf{69} & \textbf{34} & \textbf{792} & 37/\textbf{67}/\textbf{81} & 46 & \textbf{1,121} & \textbf{22}/\textbf{57}/\textbf{82} & \textbf{33} & \textbf{475} \\
    \bottomrule
    \end{tabular}
\end{table}

The comprehensive evaluation results are presented in Table~\ref{tab:all_results_pred} for GALEN, GO, and Anatomy, respectively. Our method \ontr consistently outperforms existing approaches across all datasets. While some geometric model-based methods achieve comparable performance in terms of H@$k$, such as ELEM in the GALEN dataset, they typically exhibit substantially lower average performance, which is evidenced by the significant gap between MRR and MR values. For instance, the best-performing geometric-based method on GO, Transbox, yielded MR values approximately 7 times worse and MRR values twice as poor as \ontr. This indicates that \ontr has overall fewer extreme worst cases (i.e., correct answers with extremely large ranks) and also better cases (i.e., lower rankings). This phenomenon occurs consistently across all three ontologies.

For language model-based methods, we observe that OPA2Vec and OWL2Vec* demonstrate limited performance, which is reasonable given their reliance on word embeddings and random forest %logistic regression 
classifiers for capturing subsumptions. 
It is important to note that, in our evaluation, ranking is performed over all atomic concepts as candidates---unlike the original OWL2Vec settings \cite{DBLP:journals/ml/ChenHJHAH21}, which consider only around 50 candidates---making metrics such as Hits@k not directly comparable. 
In contrast, by employing more advanced BERT-based language models and geometric embeddings based on hyperbolic spaces, HiT achieved significantly better performance. By further incorporating the logical constraints of complex concepts, our method \ontr outperformed HiT, especially in terms of average performance as indicated by MR values. 
Specifically, \ontr achieved approximately 14 times better MR values than HiT in the GO dataset, suggesting that  \ontr could more effectively avoid extremely poor cases while also improving performance on other metrics such as H@$k$ and MRR. 
Moreover, we can see that, in most cases, adding role embeddings and extra loss for logical constraints of $\exists$  and $\sqcap$ lead to better performance by comparing the \ont and \ontr.

\begin{table}[t]
    \centering
    \caption{Performance of the inference task across datasets. Values for H@k and MRR are percentages. $k=1/10/100$ for H@k.}
    \label{tab:inference_lms}
    \begin{tabular}{lccccccccc}
    \toprule
    & \multicolumn{3}{c}{\textbf{GALEN}} & \multicolumn{3}{c}{\textbf{GO}} & \multicolumn{3}{c}{\textbf{ANATOMY}} \\
    \cmidrule(lr){2-4}\cmidrule(lr){5-7}\cmidrule(lr){8-10}
    \textbf{Method} & \textbf{H@k} & \textbf{MRR} & \textbf{MR} & \textbf{H@k} & \textbf{MRR} & \textbf{MR} & \textbf{H@k} & \textbf{MRR} & \textbf{MR} \\
    \midrule
    ELEM & 0/3/9 & 1 & 8,639 & 0/3/17 & 1 & 18,377 & 0/4/22 & 2 & 4,990 \\
    ELBE & 0/4/20 & \textbf{2} & 2,999 & 0/3/15 & 1 & 4,021 & 0/6/39 & 2 & 979 \\
    BoxEL & 0/0/3 & 0 & 11,328 & 0/0/0 & 0 & 18,186 & 0/0/0 & 0 & 8,169 \\
    Box$^2$EL & 0/3/15 & 1 & 5,530 & 0/1/7 & 1 & 11,801 & 0/1/7 & 1 & 11,801 \\
    TransBox & 0/2/6 & 1 & 7,111& 0/2/9&1& 4,449& 0/5/27 & 2 & 749\\
    \midrule
     OPA2Vec & 0/0/1 &0&12,722&3/5/6& 3& 95,755&\textbf{2}/6/15&\textbf{3}&5,143\\
    OWL2Vec* &  0/0/1&0&12,647& 3/7/8 &  \textbf{5} & 88,614&1/5/14&\textbf{3}&5,441\\
    HiT& 0/4/26& \textbf{2} & 953 & 0/1/4&0&44,253 & 0/6/\textbf{44} & \textbf{3} & \textbf{441} \\
    \midrule
     \ont & 0/4/20 & 1 & 1,047 & 0/5/39 & 2 & \textbf{824} & 0/\textbf{7}/40 & \textbf{3} & 499 \\
    \ontr & 0/\textbf{5}/\textbf{28} & \textbf{2} & \textbf{913} & 0/\textbf{10}/\textbf{40} & 3 & 832 & 0/6/41 & \textbf{3} & 458 \\
    \bottomrule
    \end{tabular}
\end{table}

\subsection{Inference Task}\label{sec:inference}
The overall results are summarized in Table \ref{tab:inference_lms}. We can see that     \ontr clearly outperforms existing geometric model-based embedding methods across all datasets. In particular, on the GO dataset, our method achieves approximately 3 times better H@10 and H@100, and 5 times better MR. This improvement may reflect the advantages of hyperbolic embeddings, as HiT also shows strong results in the GALEN and ANATOMY datasets. 
However, HiT performs poorly on the GO dataset. This suggests that the information encoded in NF2–NF4 axioms, which are more prominent in GO than in other datasets, plays a crucial role. Since HiT does not incorporate this information during training, its performance suffers; such trend has also been reflected in the MR metrics on the prediction task. 
Furthermore, we observe that in most cases, incorporating role embeddings and losses for logical constraints allows \ont to achieve even better performance than both \ontr and HiT. 

The overall performance of OPA2Vec and OWL2Vec* is lower than that of most other methods. This is expected, as both OPA2Vec and OWL2Vec* are prediction-based approaches that evaluate axioms using a binary classifier. Such methods struggle to capture complex logical relationships—like the transitivity of SubClassOf relations—which limits their effectiveness in inference tasks.

% \begin{remark}
% The observed high MR value in both cases is largely due to the extensive candidate space (e.g., >180,000 in GO), where a few extreme outlier cases could contribute disproportionately to the MR. In realistic applications, we expect a much smaller candidate set—either naturally constrained or pre-filtered via techniques like ontology modularization—leading to significantly better performance than reflected by the MR values reported in our experiments.
% \end{remark}

\subsection{Other Results}

\begin{table}[t]
    \centering
    \caption{Ablation study for prediction and inference tasks on GALEN.}
    \label{tab:ablation_study}
        \begin{tabular}{l l l c c c c r}
        \toprule
        \textbf{Task} & \textbf{Method} & \textbf{Language Model} & \textbf{H@1} & \textbf{H@10} & \textbf{H@100} & \textbf{MRR} & \textbf{MR} \\ 
        \midrule
        \multirow{6}{*}{Prediction} & \multirow{3}{*}{\ont}  
                    & all-MiniLM-L6-v2  & 26 & 47 & 64 & 33 & 1,972 \\
        & &          all-MiniLM-L12-v2 & 26 & 46 & 64 & 33 & 1,546 \\
        & &          all-MPNet-base-v2 & 26 & 43 & 63 & 32 & 979 \\
        \cmidrule(lr){2-8}
        & \multirow{3}{*}{\ontr} 
                   & all-MiniLM-L6-v2 & 26 & \textbf{51} & 69 & 35 & 1,131 \\
        & &          all-MiniLM-L12-v2 & 25 & 50 & 69 & 34 & 792 \\
        & &          all-MPNet-base-v2 & \textbf{27} & 50 & \textbf{72} & \textbf{34} & \textbf{630} \\
        \midrule
        \multirow{6}{*}{Inference} & \multirow{3}{*}{\ont}  
        & all-MiniLM-L6-v2  & 0 & \textbf{5} & 23 & 2 & 1,062 \\
& &         all-MiniLM-L12-v2 & 0 & 4 & 20 & 1 & 1,047 \\
& &          all-MPNet-base-v2    & 0 & 2 & 13 & 1 & 1,148 \\
\cmidrule(lr){2-8}
& \multirow{3}{*}{\ontr} 
        & all-MiniLM-L6-v2  & 0 & 3 & 26 & 2 & 923 \\
& &          all-MiniLM-L12-v2 & 0 & \textbf{5} & 28 & 2 & 913 \\
& &          all-MPNet-base-v2    & 0 & 4 & \textbf{34} & 2 & \textbf{630} \\
        \bottomrule
        \end{tabular}
    \end{table}

\subsubsection{Ablation Study}\label{sec:ablation_study}
To evaluate the impact of different language models and loss functions, we follow the methodology of \cite{hit} and experiment with two additional top-performing pre-trained models from the Sentence Transformers library: all-MiniLM-L6-v2 (22.7M parameters) and all-mpnet-base-v2 (109M parameters), in addition to the all-MiniLM-L12-v2 model (33.4M parameters) used in our main experiments. 
From Table \ref{tab:ablation_study}, we can see that the performance differences among these models are relatively small. While the larger model consistently shows better average performance, as indicated by improved MR values, it does not always outperform the smaller models across all evaluation metrics.

\subsubsection{Transfer Learning}\label{sec:transfer_learning}  
We evaluated the OnT and HiT models in a transfer learning paradigm, where each model was trained and evaluated on a source dataset, then tested on a distinct target dataset, using three different datasets in the prediction task. 
The overall transfer learning performance of \ontr and HiT using MiniLM-L12-v2 is illustrated by Figure \ref{fig:target_grouped}. We can see that \ontr and HiT both achieve good transfer abilities, while \ontr performance better, indicated by the consistently lower MR value, and the higher H@100 or MRR value in most of cases. Especially for the cases from GO to GALEN or  ANATOMY.

\begin{figure}[t]
\centering
\includegraphics[width=0.9\textwidth]{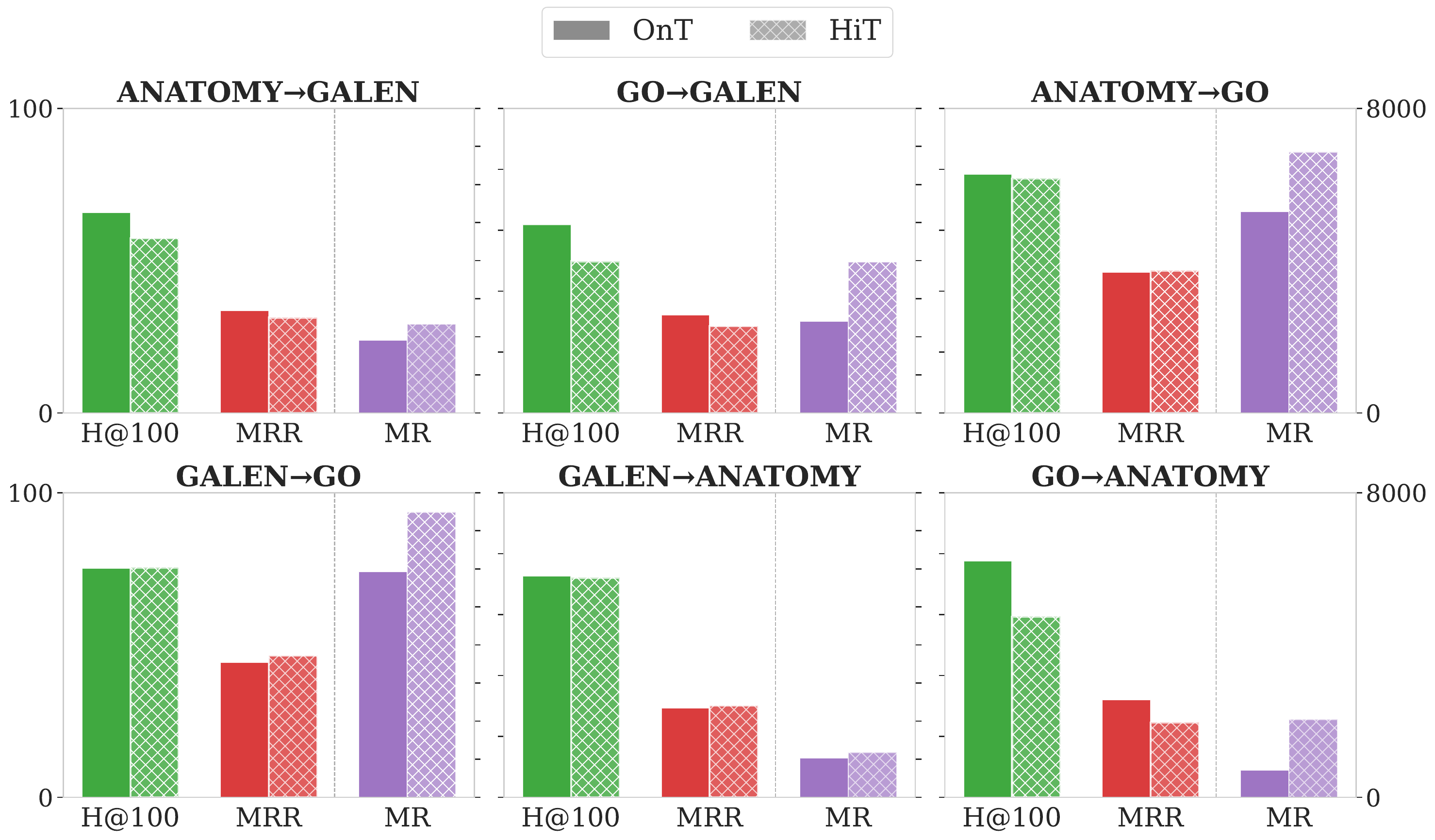}
\caption{Transfer learning results of \ontr and HiT with MiniLM-L12-v2.}
\label{fig:target_grouped}
\end{figure}

\subsubsection{Case Study}\label{sec:case_study} In our case study, we evaluate real-world scenarios encountered during the construction of ontologies, particularly in the development of a new anatomy ontology derived from SNOMED CT. The following two cases, summarized in Figure \ref{tab:case_study}, illustrate the potential of our model as a valuable tool in ontology construction.

\begin{enumerate}
    \item \textit{Missing Subsumptions:} In the manually constructed ontology of SNOMED CT, a direct subsumption is overlooked: ``Stomach structure $\sqsubseteq$ Digestive organ structure''. Our method is proven effective in identifying this missing subsumption, consistently assigning it a higher score than other existing superclasses of the ``Stomach structure'' within the constructed ontology.
    \item \textit{Erroneous (Direct) Subsumptions:} We detect an incorrect direct Superclass of ``Bone structure of upper limb'' as ``Structure of appendicular skeleton'', which is incorrect as ``Bone structure of extremity'' should have such a parent. Our model effectively identifies this erroneous relationship by consistently assigning it the lowest score among all existing superclasses.
\end{enumerate}

    \begin{figure}[t]
        \centering
        \begin{tikzpicture}[>=stealth, node distance=2.5cm,
          concept/.style={draw, rounded corners, fill=gray!10, minimum width=3.5cm, text width=3.8cm, align=center},
          narrow_concept/.style={draw, rounded corners, fill=gray!10, minimum width=2.5cm, text width=2cm, align=center},
          missing/.style={draw=green!50!black, very thick, dashed},
          incorrect/.style={draw=orange!50!black, very thick, dotted},
          normal/.style={draw=black, thick},
          lbl/.style={fill=white, font=\footnotesize, align=center},
          scale=0.95
        ]
        
        % Case 1: Missing Hierarchy - Parents on top, children below
        \node[narrow_concept] (dos) at (-4.5,-2) {Digestive organ structure};
        \node[narrow_concept] (sds) at (-1.8,-2) {Stomach and/or duodenal structures};
        \node[narrow_concept] (egs) at (1,-2) {Esophageal and/or gastric structures};
        \node[narrow_concept] (soac) at (4,-2) {Structure of organ within abdomen proper cavity};
        \node[concept] (ss) at (-2,-4.5) {Stomach structure};
        
        % Missing edge - highlighted (upward arrow from child to parent)
        \draw[->, missing] (ss) to[bend left=30] node[lbl, near end] {{\color{blue}-8.1/-5.5/-4.6}} (dos);
        
        % Existing edges (upward arrows from child to parents)
        \draw[->] (ss) to[bend right=20] node[lbl, near end ] {{\color{red}-13.3/-10.8/-12.2}} (egs);
        \draw[->] (ss) to[bend right=25] node[lbl, midway ] {-11.4/-9.7/-10.8} (soac);
        \draw[->] (ss) to[bend left=15] node[lbl,near start] {-10.7/-8.4/-9.3} (sds);
        
        % Separate the two cases
        \draw[dashed, gray] (-6,-5.1) -- (6,-5.1);
        
        % Case 2: Erroneous Hierarchy - with proper hierarchy
        \node[concept] (sas) at (3,-5.7) {Structure of appendicular skeleton};
        \node[concept] (msul) at (-3,-7.2) {Musculoskeletal structure of upper limb};
        \node[concept] (bse) at (3.5,-7.2) {Bone structure of extremity};
        \node[concept] (bsul) at (0,-9) {Bone structure of upper limb};
        
        % Incorrect edge - highlighted (upward arrow from child to parent)
        \draw[->, incorrect] (bsul) to[bend left=35] node[lbl] {{\color{red}-10.9/-6.8/-7.5}} (sas);
        
        % Correct edges (upward arrows from child to parents)
        \draw[->] (bsul) to[normal] node[lbl] {{\color{blue}-6.7/-5.9/-6.7}} (msul);
        \draw[->] (bsul) to[normal] node[lbl] {-10.0/-6.5/{\color{blue}-6.7}} (bse);
        
        % Adding hierarchical relationship between SAS and BSE
        \draw[->, normal] (bse) to (sas);
        
        % Legend - all in one horizontal line
        \node at (0,-10) {
          \begin{tabular}{cccccc}
            \begin{tikzpicture}[baseline=-0.5ex] \draw[->,missing] (0,0) -- (0.8,0); \end{tikzpicture} & Missing subsumption &
            % add emptyspace 
            \ \ \ \ \ \ \ \ 
            \begin{tikzpicture}[baseline=-0.5ex] \draw[->,incorrect] (0,0) -- (0.8,0); \end{tikzpicture} & Non-direct subsumption &
            %\begin{tikzpicture}[baseline=-0.5ex] \draw[normal] (0,0) -- (0.8,0); \end{tikzpicture} & subsumption \\
          \end{tabular}
        };
        \end{tikzpicture}
        % \caption{Subsumptions are represented by arrows with three scores from \ontr models trained on GALEN/GO/ANATOMY. Blue and red highlights indicate highest and lowest scores.}
        % \label{tab:case_study}
        \caption{Case Study: Arrows $C\rightarrow D$ represent subsumption $C\sqsubseteq D$ with scores from Eq. \ref{eq:scoring}. Each arrow shows three scores from three \ontr models trained on GALEN/GO/ANATOMY ontologies, respectively. A higher score indicates a more likely subsumption. Blue/red highlights indicate the highest/lowest scores for all subsumptions with the same subclass.}
        \label{tab:case_study}
        \end{figure}
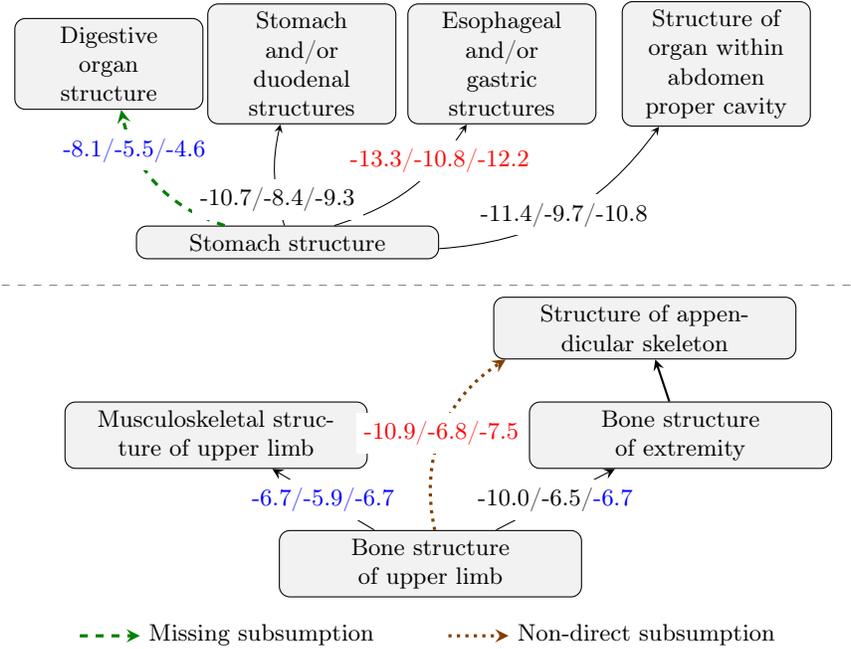

\section{Conclusion and Future Work}
In this study, we introduce \ontr, which integrates geometric models with language models to derive ontology embeddings for concepts and roles. Through extensive experiments on real-world ontologies, we demonstrate that our approach achieves state-of-the-art performance in both prediction (inductive reasoning) and inference (deductive reasoning) tasks. Furthermore, our method exhibits strong transfer learning capabilities, suggesting its potential for real-world applications in related domains.  

Looking ahead, our future research aims to merge our current methodologies with other hierarchical embedding techniques, such as \cite{EntCone,yang2025regd}. Additionally, we are keen to extend our methods to more complex ontology languages, such as extending to $\mathcal{ALC}$ with the negation logical operator $\neg$, or delve deeper into the logical patterns of roles using the role embeddings generated by \ontr, including investigating role inclusion axioms. It would also be interesting to conduct a more thorough analysis of our model, such as exploring the impact of verbalization quality, or performance across a wider range of ontologies beyond those currently utilized.

\paragraph{Supplementary Materials} All supplementary materials, including the code and dataset, are available at \url{https://github.com/HuiYang1997/OnT}.

\bibliographystyle{splncs04}
\bibliography{ontE}
\end{document}